\def\year{2018}\relax
\documentclass[letterpaper]{article} 
\usepackage{aaai18}  
\usepackage{times}  
\usepackage{helvet}  
\usepackage{courier}  
\usepackage{url}  
\usepackage{graphicx}  
\usepackage{wasysym} 
\usepackage{verbatim} 
\usepackage{amsfonts} 
\usepackage{amssymb} 
\usepackage{amsthm} 
\usepackage[shortlabels]{enumitem} 

\newtheorem{theorem}{Theorem}[section]
\newtheorem{example}[theorem]{Example}
\newcommand{\U}{{\mathcal U}}
\newcommand{\V}{{\mathcal V}}
\newcommand{\R}{{\mathcal R}}
\newcommand{\G}{{\mathcal G}}
\newcommand{\F}{{\mathcal F}}
\newcommand{\M}{{\mathcal M}}
\newcommand{\commentout}[1]{}
\newtheorem{proposition}[theorem]{Proposition} 
\newtheorem{definition}[theorem]{Definition} 
\renewcommand{\citeyear}{\shortcite}
\newcommand{\wbox}{\mbox{$\sqcap$\llap{$\sqcup$}}}

\begin{document}
%
\title{Combining the Causal Judgments of Experts with Possibly Different Focus Areas}
\author{Meir Friedenberg \\
	Department of Computer Science\\
	Cornell University\\
	meir@cs.cornell.edu 
	\And 
	Joseph Y. Halpern\\
	Department of Computer Science\\
	Cornell University\\
	halpern@cs.cornell.edu\\
}
\maketitle
\begin{abstract}
In many real-world settings, a decision-maker must combine
information provided by different experts in order to decide
on an effective policy.  Alrajeh, Chockler, and Halpern
\citeyear{ACH18}
showed how
to combine causal models that are compatible in the sense that, for variables
that appear in both models, the experts agree on the causal structure.
In this work we show how causal models can be combined in
cases where the experts might disagree on the causal structure for variables
that appear in both models due to having different focus areas.
We provide a new formal definition of compatibility of
models in this setting and show how compatible models can be
combined.   We also consider the complexity of determining whether models
are compatible.  We believe that the notions defined in this work are of
direct relevance to many practical decision making scenarios that
come up in natural, social, and medical science settings.
\end{abstract}

\section{Introduction}
In many real-world settings, a decision-maker must combine information
provided by different experts in order to decide on an
effective policy.  For example, when deciding policing and criminal
   justice policy, it may be necessary to consult different experts
specializing in areas such as criminology, psychology, sociology, and
economics.  Intelligently combining the information provided by the
various experts is necessary if the decision-maker hopes to select the
best course of action. 

Much work has been done on combining simple probabilistic judgments of
different experts.
However, 
we are interested in settings
where a decision-maker wants  to choose an action in order to induce
a particular outcome, so we 
are interested in the setting where experts provide models of
the \textit{causal} relationships between different factors.  Despite
the clear importance of combining causal models in real-world
situations, there has been very little
work on how to combine models with this extra structure.

Much work has been done on the related problem of learning causal
models: given data and possibly some prior structured knowledge,
extract the causal model that best fits the given information
(see, e.g., \cite{CH10,CH12,HEJ14,TS11,TT15}; Triantafillou and
Tsamardinos \shortcite{TT15} provide a good overview of work in the
area).
In a
certain sense, if all of the experts we consult have learned their
models in this manner and can provide us with all of their data, then
the best thing to do is simply to learn a new causal model from the
union of all the data.  However, in many real world settings, this is
completely impractical.   Experts develop intuitions based on years
worth of data, training, and discussions; providing all of this
background information to the decision maker may be infeasible.

On the topic of combining causal models without data, Bradley,
Dietrich, and List \citeyear{BDL14}
proved an impossibility result.  Given a set of desiderata
for combining causal models, they show that no algorithm 
satisfies them all.  They further examine ways of
circumventing the impossibility result by weakening some of those
conditions. 

It is perhaps not too surprising in retrospect that it will sometimes
be impossible to combine causal models, as two models can explicitly
disagree on every causal relationship.  In the work most related to
this, Alrajeh, Chockler, and Halpern \citeyear{ACH18} (ACH from now on)
provide conditions for
the \textit{compatibility} of models and show how to combine models
that meet their compatibility conditions.  They define a dominance
relation according to which a model $M_1$ dominates model $M_2$ with
respect to a 
variable $C$ if the two models agree on the
causal dependence of $C$ on the other variables shared by the two
models, but model $M_1$ has perhaps a more detailed picture of the
exact way that the effects are mediated.  Two models are compatible if, for
every variable, one of the models
dominates the other; the combined model takes the causal information
from the dominant model for each variable $C$.  ACH also provide a
way of assigning probabilities to causal models in settings where not
all models under consideration are compatible. 

The present work can be seen as providing an approach complementary to
that of ACH.  Philosophically,
the approach presented by ACH is intended to allow for
combination of models where the modelers fundamentally agree on the
causal relationship between the variables they both
discuss, but go into different levels of detail as to how some of those
relationships are mediated.  But consider, for instance, the following
scenario: a medical scientist is interested in the conditions under
which a particular reaction occurs, and consults with two experts.  The
first specializes in the exact mechanism by which
this reaction occurs; the second specializes 
in how one of the reactants gets produced.
Because of their different focus areas, they in fact do not agree on
everything; each of them is more aware of the details of the reaction that she
studies, and thus has more understanding of what factors can cause
that reaction to occur differently.  Our intuition tells us that there
should be a way of combining these models to get the true expertise of
both modelers, but with the ACH approach, these
models would in fact have to be deemed incompatible. 

In this work, we allow for combining models where the modelers
disagree due to their different focus areas.  Intuitively, if the first
modeler considered more possibilities than the second, and her
 conclusion can \textit{explain} the observations of the second, then
 we accept the conclusion of the first modeler.  To this end, we
 define a new notion of a  \textit{``can explain''} relation and
 provide new formalizations of compatibility and model combination
 relative to this notion.   


The rest of this paper is organized as follows.  In Section~\ref{sec:focus}, we
review the basic framework of causal models, and extend them so as to
accommodate focus areas.
In Section~\ref{sec:combine}, we define our approach
to combining these models. Section~\ref{sec:weight} contains an approach to
weighting models in settings where the models under consideration are
not all compatible.
We characterize the computational complexity of the can-explain
relation that we define in Section~\ref{sec:complexity}.
Section~\ref{sec:conclusion} concludes.

\section{Causal Models with Focus}
\label{sec:focus}

In this section, we review the framework of causal models.  We largely follow
Halpern and Pearl \citeyear{HP01b},
but extend their basic framework so as to allow the models to express
focus areas.

We assume that a situation is characterized by the values of a number
of variables.  There are
\textit{structural equations} describing the effect that the
variables have on each other.  Among the variables, we distinguish
between \textit{exogenous variables} (whose values are determined by
factors outside of the model) and \textit{endogenous variables} (whose
values are determined by other variables in the model). 

A \textit{causal model with focus} is a tuple $M =
(\mathcal{S},\mathcal{F},\mathcal{G})$, where $\mathcal{S}$ is
a \textit{signature},
$\mathcal{F}$ is a set of \textit{structural equations},
and $\mathcal{G}$ is a \textit{focus function}.  The
signature $\mathcal{S}$ is itself a tuple
$(\mathcal{U},\mathcal{V},\mathcal{R})$.  Here $\mathcal{U}$ is a
(finite but non-empty) set of exogenous variables and $\mathcal{V}$ is
a (finite but non-empty) set of endogenous variables.  $\mathcal{R}$
is a \textit{range} function mapping elements of
$\mathcal{U} \cup \mathcal{V}$ to the (finite but non-empty) set of
values they can take on.
We assume without loss of generality that $|\mathcal{R}(C)|>1$ for all
variables $C$.  (If a variable can take on only one value,  
then it can neither be a cause nor have its value be
caused by another variable, so we can remove it and get a semantically
equivalent model.)
$\F$ associates with each endogenous variable $X \in \V$ a
function denoted $F_X$ such that $F_X: (\times_{U \in \U} \R(U))
\times (\times_{Y \in \V - \{X\}} \R(Y)) \rightarrow \R(X)$;
that is, $F_X$ determines the value of $X$,
given the values of all the other variables in $\U \cup \V$.
For example, we might have 
$F_X(u,y,z) = u+y$, which is usually written as
$X = U+Y$.   Thus, if $Y = 3$ and $U = 2$, then
$X=5$, regardless of how $Z$ is set.

Up to now, we have essentially described the
Halpern-Pearl \citeyear{HP01b} model.  
In our setting, though, we add an additional \textit{focus function}
that intuitively tells us what variables the modeler considered when
trying to determine the structural equation for each variable.  In
practice, we might extract such information from the modeler herself
or from the published experiments of the modeler.  Formally, we let
$\mathcal{G}:(\mathcal{U} \cup \mathcal{V}) \rightarrow
2^{(\mathcal{U} \cup \mathcal{V})}$ be a function that, given a
variable $C$, gives us the set of variables that the modeler considered as
possibly having an effect on $C$.  
We require for all $C$ that $C \notin \mathcal{G}(C)$, as a variable
cannot have a causal effect on itself.  
It may seem surprising at first
that we define this function even for exogenous variables, which are
not affected by other variables in the model.  We think that it is more
natural to do so, as it is possible that the modeler considered the
possibility of the variable being endogenous before deciding that it wasn't
affected by any other variables in the model. 

We define $B$ to be a \emph{parent} of $C$ if
there exists some setting of all the variables in
$(\mathcal{V} \cup \mathcal{U}) - \{B,C\}$ such that $C$ takes on some
value $c_1$ for a value $b_1$ of $B$ and takes on a different value
$c_2$ for some other values $b_2$ of $B$.  Let $Par(C)$ be the set of
parents of $C$.
Thus, the parents of $C$ are exactly those variables that might have a
direct effect on $C$. 
We require that $Par(C) \subseteq \mathcal{G}(C)$ for
every variable $C$.    
A modeler cannot have an equation for $C$ showing that $B$ has a
direct influence on $C$ unless the modeler considered $B$ as a
possibly having an effect on $C$.

 A causal model with focus with exogenous variables $\U$ and
endogenous variables $\V$ can be represented by a pair of graphs on
$\mathcal{U} \cup \mathcal{V}$.  In the first graph,
called the \emph{parent graph},
the edge set $E$ consists of edges from the vertices in $Par(C)$ to
$C$, for each 
endogenous variable $C$.
In the second graph,
called the \emph{focus graph},
the edge set $E'$ consists of edges to each vertex $C$ from
the members of $\mathcal{G}(C)$.
Pictorially, we can depict this
representation with directed edges for the elements of $E$ and
crossed-out directed edges for the elements of $E'-E$. We call a
model \textit{recursive} or \textit{acyclic} if the
parent graph does not
contain any cycles.  In cases where the model is acyclic, given a
context $\vec{u}$ (i.e., a setting of the exogenous variables), the
values of all the endogenous variables are uniquely determined by the
structural equations.  As is standard in the
literature, we restrict our discussion to acyclic models in this
work. 

\commentout{
Given a signature $\mathcal{S} =
(\mathcal{U},\mathcal{V},\mathcal{R})$, a primitive event is a formula
of the form $C=c$ where $C \in \mathcal{U}\cup{\mathcal{V}}$ and
$c \in \mathcal{R}(C)$.  A boolean combination $\varphi$ of primitive
events is true or false relative to a \textit{(causal) setting}, a
pair $(M,\vec{u})$ consisting of a model and a context.  We write
$(M,\vec{u}) \vDash \varphi$ if $\varphi$ is true in the setting
$(M,\vec{u})$.  $\vDash$ is formally defined inductively, with a
primitive event $C=c$ being true if $C$ is assigned the value $c$ in
the (unique, because the models under consideration are acyclic)
solution to the structural equations under context $\vec{u}$. 

\subsection{Interventions}
}

Given a model $M$, an endogenous variable $X \in \mathcal{V}$, 
and a value $x \in \mathcal{R}(X)$,
we define $M_{X \leftarrow x}$ to be the model 
that is the same as $M$ except that the equation for
$X$  is replaced by $X=x$.  We can think of
the model $M_{X \leftarrow x}$ as describing the result of 
intervening to set $X$ to $x$ in model $M$.

%
%

Take a \emph{causal formula} to be one of the form $[Y_1 \gets
y_1, \ldots,  Y_k \gets y_k]\varphi$, where $Y_j \in \U \cup \V$ and
$\varphi$ is a Boolean combination of \emph{primitive formulas} of the
form $X=x$, where $X$ is an endogenous variable and $x \in \R(X)$.%
\footnote{In previous work, each $Y_i$ is taken to be an endogenous
variable.  For our purposes, it is useful to also allow $Y$ to be
exogenous.}  In the special case where $k=0$, we identify $[\,]\varphi$
with the formula $\varphi$.

We now define what it means for a causal formula $\varphi$ to be true in
a \emph{causal setting} $(M,\vec{u})$ consisting of a causal 
model $M$ and a context $\vec{u}$, written
$(M,\vec{u}) \models \varphi$,  by induction on the structure of $\varphi$.
For a primitive event $X=x$, $(M,\vec{u}) \models X=x$ if $X=x$ 
in the unique solution to the equations in $M$ given context $\vec{u}$
(the solution is unique since we are dealing with acyclic models, so
the setting of the exogenous variables determines all other
variables).  The truth of a Boolean combination of primitive events is
defined in the obvious way.  If $k \ge 1$ and $Y_k$ is an endogenous
variable, then
$$\begin{array}{ll}
(M,\vec{u}) \models [Y_1 \gets y_1, \ldots, Y_k \gets
y_k]\varphi \mbox{ iff } \\
(M_{Y_k \gets y_k},\vec{u}) \models
[Y_1 \gets y_1, \ldots, Y_{k-1} \gets y_{k-1}]\varphi.\end{array}$$
If $Y_k$ is an exogenous variable, then
$$\begin{array}{ll}
(M,\vec{u}) \models [Y_1 \gets y_1, \ldots, Y_k \gets
y_k]\varphi \mbox{ iff }\\ (M,\vec{u}[Y_k/y_k]) \models
[Y_1 \gets y_1, \ldots, Y_{k-1} \gets y_{k-1}]\varphi,\end{array}$$
where $\vec{u}[Y_k/y_k]$ is the result of replacing the value of $Y_k$
in $\vec{u}$ by $y_k$.

We now show how to model the example from the introduction in this framework.
\begin{example}\label{ex1}
{\rm 
Recall the basic scenario: a medical scientist is trying to
 understand under what 
conditions a particular reaction occurs, and consults with two
experts.  The first specializes in the
exact mechanism by which this reaction occurs; the second
in how one of the reactants gets produced.
The scientist then wants to 
combine the information provided by the two experts. 
The models provided by the experts are depicted 
in Figure~\ref{fig:twoscientistsexample1}, where expert $i$ provides
model $M_i$.
The main difference between
these models is that expert $1$ takes into account the effect
that temperature $T$ can have on reaction $C$, while modeler $2$, who
does not, takes into account the effect temperature can have on the
production of reactant $B$.  Formally, the parameters of these two
models are defined as follows: for the ranges, we have
$\mathcal{R}_{1}(T) = \mathcal{R}_{2}(T)
= \{\mathsf{Freezing},\mathsf{Cool},\mathsf{Hot}\}$,
$\mathcal{R}_{1}(A')=\mathcal{R}_{1}(B')=\mathcal{R}_{2}(A')=\mathcal{R}_{2}(B')
= \{1,2,3\}$,
$\mathcal{R}_{1}(A)=\mathcal{R}_{1}(B)=\mathcal{R}_{2}(A)=\mathcal{R}_{2}(B)
= \{1,\dots,5\}$, and $\mathcal{R}_{1}(C)=\mathcal{R}_{2}(C)
= \{\mathsf{true},\mathsf{false}\}$.
We have $\G_1(A) = \{A'\}$, $\G_1(B) = \{B'\}$, and $\G_1(C) = \{A,B,T\}$, while
 $\G_1(A) = \{A'\}$, $\G_1(B) = \{B',T\}$, and $\G_1(C) = \{A,B\}$.
This is how we model the fact that expert 1 does not take into account
 the effect that temperature ($T$) can have on $B$, while expert 2 does
 not take into account the effect that temperature can have on $C$.
The structural equations 
in $M_1$ are defined by taking $A = A'$, 
$B = B'$, and $C = 
((T = \mathsf{Freezing}) \wedge (A+B \geq 9)) \vee
((T = \mathsf{Cool}) \wedge (A+B \geq 5)) \vee
((T=\mathsf{Hot}) \wedge (A+B \geq 4))$.  In $M_2$, the structural
equations are $A = A'$; $B = B' + 2$ if $T = \mathsf{Freezing}$ and
$B= B'$ otherwise; and 
$C = \mathsf{true}$ if $A+B \geq 5$ and $C= \mathsf{false}$ otherwise.
}
\begin{figure}[htb]
\begin{center}
\includegraphics[width=1.0\linewidth]{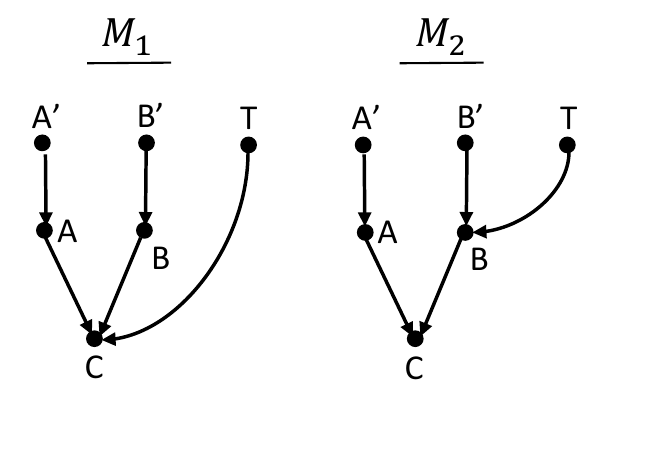}
\end{center}
\vspace{-.5in}
\caption{$M_1$ and $M_2$ are the models of the two scientists trying
to understand reaction $C$.} 
\label{fig:twoscientistsexample1}
\end{figure}
\wbox
\end{example}

\commentout{
Given a signature $\mathcal{S} =
(\mathcal{U},\mathcal{V},\mathcal{R})$, a primitive event is a formula
of the form $C=c$ where $C \in \mathcal{U}\cup{\mathcal{V}}$ and
$c \in \mathcal{R}(C)$.  A boolean combination $\varphi$ of primitive
events is true or false relative to a \textit{(causal) setting}, a
pair $(M,\vec{u})$ consisting of a model and a context.  We write
$(M,\vec{u}) \vDash \varphi$ if $\varphi$ is true in the setting
$(M,\vec{u})$.  $\vDash$ is formally defined inductively, with a
primitive event $C=c$ being true if $C$ is assigned the value $c$ in
the (unique, because the models under consideration are acyclic)
solution to the structural equations under context $\vec{u}$. 
}

\section{Combining Causal Models with Focus}
\label{sec:combine}

In this section we turn to the question of combining causal models.
We define a new relation and show how it can be used to define
compatibility and combination. 

\subsection{The ``can-explain'' relation}

We want to capture the intuition that if modeler $i$ considered
the causes of some variable $C$ more carefully than modeler $j$, then
$i$'s analysis is preferable.  Roughly speaking, we prefer
modeler $i$'s structural equation for $C$ over $j$'s if 
$i$'s model can explain (in some appropriate sense) $j$'s
observations.

Before going on, we introduce some notation conventions that will
simplify the exposition.  When we write $M_i$, we assume that the
model $M_i$ has components $((\U_i,\V_i,\R_i),\F_i,\G_i)$, and
$Par_i(C)$ refers to the parents of variable $C$ in model $M_i$.
Also, for a model $M$, we write $\F^M$ to denote
the $\F$ function in model $M$, and similarly for the other components
of the model.

\begin{definition}\label{dfn:canexplain} {\rm $M_1$
\textit{can explain} $M_2$
with respect to $C$,
written $M_1 \succeq_C M_2$, if  
\begin{itemize}
\item[(a)] $\R_1(C) = \R_2(C)$,
\item[(b)] $\mathcal{G}_2(C) \subseteq \mathcal{G}_1(C)$, and
\item[(c)] for all exogenous settings $\vec{u}_2$ for $M_2$ and
all interventions $\mathcal{G}_2(C) = \vec{x}$ 
there is a context $\vec{u}_1$ in $M_1$
such that if $(M_2, \vec{u}_2) \vDash
[\mathcal{G}_2(C) \leftarrow \vec{x}]( C=c)$ then
$(M_1, \vec{u}_1) \vDash [\mathcal{G}_2(C) \leftarrow \vec{x}] (C=c)$. 
\end{itemize}
}
\end{definition}

This relation ought really be called \textit{can explain and has
considered everything considered by}, but for the sake of brevity we
use simply \textit{can explain}. 

The intuition here is that expert 2 (whose knowledge is characterized
by $M_2$) has considered carefully the effect on $C$ of all the variables in
$\G_2(C)$ and has observed that those in $Par_2(C)$ have an effect on
$C$, while those in $\G_2(C) - Par_2(C)$ do not.  She has not bothered
considering the effect of the variables not in $\G_2(C)$ on $C$,
because she is reasonably sure that they have no effect (but could turn
out to be wrong about this).  Expert 1 (whose knowledge is
characterized by $M_1$) can explain expert 2's observations (at least,
with regard to $C$) if she has also considered at least all of the
interventions that expert 2 has considered, and can explain all of expert
2's observations in the sense of condition (c) of
Definition~\ref{dfn:canexplain}.


We conclude this subsection with two technical results that highlight
useful properties of the $\succeq_C$ relation.
Say that $M_1$ and $M_2$ are \emph{$C$-compatible} if either
$M_1 \succeq_C M_2$ or $M_2 \succeq M_1$.  We now show that
$\succeq_C$ is transitive when restricted to $C$-compatible models.
\begin{proposition}\label{pro:preorder}
If $M_1 \succeq_C M_2$, $M_2 \succeq_C M_3$, and $M_1$ and $M_3$ are
$C$-compatible, then $M_1 \succeq_C M_3$.  
\end{proposition}

\begin{proof}
%
Assume that $M_1 \succeq_C M_2$,
$M_2 \succeq_C M_3$, and by way of contradiction, that
	$M_1 \not\succeq_C M_3$.
Because $M_1$ and $M_3$ are $C$-compatible, it must then be the case
	that $M_3 \succeq_C M_1$. 
Since $\R_1(C) = \R_2(C)$ and $\R_2(C) = \R_3(C)$, we have that
$\R_1(C) = \R_3(C)$.  And since 
	$\mathcal{G}_2(C) \subseteq \mathcal{G}_1(C)$,
	$\mathcal{G}_3(C) \subseteq \mathcal{G}_2(C)$, and
$\mathcal{G}_1(C) \subseteq \mathcal{G}_3(C)$, 
        we get that $\mathcal{G}_1(C) = \mathcal{G}_2(C)
	= \mathcal{G}_3(C)$.  Now consider any intervention
$\G_3(C) = \vec{x}$.  
Because $M_2 \succeq_C M_3$,
we know that any value of $C$ that can be
	achieved in $M_3$ under intervention $\vec{X}=\vec{x}$ can
also be achieved in $M_2$ under the same intervention;
that is, for all contexts $\vec{u}$ and values $c \in \R(C)$, if
$(M_3,\vec{u}) \models [\G_3(C) \gets \vec{x}](C=c)$, then there exists
a context $\vec{u}'$ such that
$(M_2,\vec{u}') \models [\G_3(C) \gets \vec{x}](C=c)$.
But	because $\mathcal{G}_2(C) = \mathcal{G}_3(C)$
and $M_1 \succeq_C M_2$,
it follows that there exists a context $\vec{u}''$ such that
$(M_1,\vec{u}'') \models [\G_3(C) \gets \vec{x}](C=c)$.
Thus,  condition (c) of Definition~\ref{dfn:canexplain} holds,
so $M_1 \succeq_C M_3$.  
%
\end{proof}

The requirement in Proposition~\ref{pro:preorder} 
that $M_1$ and $M_3$ are $C$-compatible is necessary, as we show 
below (see Example~\ref{ex:notlub}).

\begin{definition}\label{dfn:C-equivalent} {\rm 
$M_1 \equiv_C M_2$ 
iff either (a) $C \in U_1 \cap U_2$, $\R_1(C) = \R_2(C)$, and
$\mathcal{G}_1(C)=\mathcal{G}_2(C)$ or (b) $C \in V_1 \cap V_2$,
$\R_1(C) = \R_2(C)$, 
$\mathcal{G}_1(C)=\mathcal{G}_2(C)$, and
$\mathcal{F}_1(C)=\mathcal{F}_2(C)$.%
\footnote{Technically
$\mathcal{F}_1(C)=\mathcal{F}_2(C)$ is not defined if
$\mathcal{U}_1\cup\mathcal{V}_1 \neq \mathcal{U}_2\cup\mathcal{V}_2$;
all we mean is that $Par_1(C)=Par_2(C)$,
$\mathcal{R}_1(D)=\mathcal{R}_2(D)$ for all $D \in Par_1(C)$, and
$(M_1,\vec{u}_1)\vDash [Par_1(C) \leftarrow \vec{p}] (C=c)$ iff
$(M_2,\vec{u}_2)\vDash [Par_2(C) \leftarrow \vec{p}] (C=c)$ for all
$\vec{u}_1, \vec{u}_2, \textrm{ and } \vec{p}$.}
}
\end{definition}

The next result shows that, in a sense, $\succeq_C$ is anti-symmetric.

\begin{proposition}\label{pro:antisymmetry}
$M_1 \succeq_C M_2$ and $M_2 \succeq_C M_1$ iff $M_1 \equiv_C M_2$. 
\end{proposition}

\begin{proof}
The fact that $M_1 \equiv_C M_2$ implies $M_1 \succeq_C M_2$ and
$M_2 \succeq_C M_1$ follows easily from the definitions, using the
fact that $Par_i(C) \subseteq \G_i(C)$.  

To prove the opposite implication, suppose that $M_1 \succeq_C M_2$
and $M_2 \succeq_C M_1$.  We first show that $C$ cannot be 
in either $\U_1 \cap \V_2$ or $\U_2 \cap \V_1$.  Suppose, by way of
	contradiction, that $C \in \U_1 \cap \V_2$.
Consider an intervention $\mathcal{G}_2(C) = \vec{x}$.  Because
$C$ is exogenous in $M_2$, there must exist contexts $\vec{u}_2 \neq
\vec{u'}_2$ such that
$(M_2,\vec{u}_2) \vDash [\mathcal{G}_2(C) \leftarrow \vec{x}]
(C=c_2)$ and $(M_2,\vec{u}'_2) \vDash
	[\mathcal{G}_2(C) \leftarrow \vec{x}]( C=c'_2)$ for some $c_2$
	and $c'_2$ such that $c_2 \neq c'_2$.  Now consider this same
	intervention in $M_1$.
Since $M_1 \succeq_C M_2$ and $M_2 \succeq M_1$, we have that
$\mathcal{G}_1(C) = \mathcal{G}_2(C)$.  By definition,
$Par_1(C) \subseteq \mathcal{G}_1(C)$.  Thus there
must exist a unique $c_1$ such that, for all exogenous settings
	$\vec{u}_1$ in $M_1$, $(M_1,\vec{u}_1) \vDash
[\mathcal{G}_2(C) \leftarrow \vec{x}]( C=c_1)$.  But because
	$c_2 \neq c'_2$,
there cannot be contexts $\vec{u}_1$ and $\vec{u}_1'$ such that 
$(M_1,\vec{u}_1) \vDash [\mathcal{G}_2(C) \leftarrow \vec{x}](C=c_2)$
and $(M_1,\vec{u}_1') \vDash [\mathcal{G}_2(C) \leftarrow \vec{x}](C=c_2')$.
This contradicts the assumption that $M_1 \succeq_C M_2$.
A similar argument shows that $C$ cannot be in $\U_2 \cap \V_1$.

It is almost immediate from the definition of $\succeq_C$ that if
$M_1 \succeq_C M_2$, $M_2 \succeq_C M_1$, and
$C \in (\U_1 \cap \U_2) \cup (\V_1 \cap \V_2)$, then $\R_1(C)
= \R_2(C)$ and $\G_1(C) = \G_2(C)$.  
It follows that if $C \in \U_1 \cap \U_2$, then $M_1 \equiv_C M_2$.
It remains to show that if $C \in \V_1 \cap \V_2$, then $\F_1(C) = \F_2(C)$.


So suppose that $C \in V_1 \cap V_2$.
If $Par_1(C) = Par_2(C)$, then
since $\G_1(C) = \G_2(C)$ and $Par_i(C) \subseteq \G_i(C)$ for
$i=1,2$, it follows that 
$(M_1, \vec{u}_1) \vDash
[\mathcal{G}_2(C) \leftarrow \vec{x}]( C=c)$ iff
$(M_2, \vec{u}_2) \vDash
[\mathcal{G}_2(C) \leftarrow \vec{x}]( C=c)$
for all contexts $\vec{u}_1$ and $\vec{u}_2$, so
$\F_1(C) = \F_2(C)$.  On the other hand, if
$Par_1(C) \ne 
Par_2(C)$,
then without loss of generality there is some variable $D \in Par_1(C)
- Par_2(C)$.  
There must thus exist two interventions $\mathcal{G}_2(C) = \vec{x}$ and
$\mathcal{G}_2(C) = \vec{y}$ that differ only on the value of $D$
such that for some $c \in \R(C)$, we have $(M_1,u'_1) \vDash
[\mathcal{G}_2(C) \leftarrow \vec{x}](C=c)$ and
$(M_1,u'_1) \vDash
[\mathcal{G}_2(C) \leftarrow \vec{y}] \neg(C=c)$ for all exogenous
settings $u'_1$ in $M_1$.  Because $D \notin Par_2(C)$, we know that
interventions $\mathcal{G}_2(C) = \vec{x}$ and $\mathcal{G}_2(C)
= \vec{y}$ will give the same value of $C$ in $M_2$ for all
settings of exogenous variables $u_2$.
Thus, it is not the case that
$M_1$ can explain $M_2$ with respect to $C$,
giving a contradiction. 
	
So we have in all cases that $M_1 \equiv_C M_2$, as desired.
%
\end{proof}


\subsection{Combining compatible models}

We now turn to compatibility and combination of causal models.
%
We start by defining a simplified notion of compatibility and an
operator ${\oplus'}$ that gets us most of the way there.
Unfortunately, as we show, $\oplus'$ has a small shortcoming, so we then
modify it to get a more reasonable operator $\oplus$.

\begin{definition}\label{dfn:compatible}
$M_1$ and $M_2$ are \emph{compatible} if,
for all $C \in
(\mathcal{U}_1 \cup \mathcal{V}_1) \cap
(\mathcal{U}_2 \cup \mathcal{V}_2)$, either $M_1 \succeq_C M_2$ or
$M_2 \succeq_C M_1$.   
\end{definition}

If $M_1$ and $M_2$ are compatible then, for each variable $C$, we intuitively
want the combined model to take all of the information for $C$ from
the model that 
best explains $C$. So if $M_1$ can explain $M_2$ with respect to $C$,
then we want the combined model to use $M_1$'s focus function and structural
equation (if $C$ is endogenous in $M_1$) for $C$.
Formally, the combined model $M_1
{\oplus'} M_2 =
((\mathcal{U},\mathcal{V},\mathcal{R}),\mathcal{F},\mathcal{G})$ is
defined as follows:
\begin{itemize}
\item 
$\mathcal{U}\cup\mathcal{V} =
(\mathcal{U}_1\cup\mathcal{V}_1)\cup(\mathcal{U}_2\cup\mathcal{V}_2)$
(so the exogenous and endogenous variables in the combined model
comprise all the endogenous and exogenous variables in $M_1$ and $M_2$).
A variable $U$ is exogenous in $M_1 \oplus' M_2$ if it is exogenous in 
one of $M_1$ or $M_2$, say $M_i$, and either does not appear in
$M_{3-i}$ (i.e., the other model) or it appears in $M_{3-i}$ but $M_i 
\succeq_U M_{3-i}$; the remaining variables are endogenous.
Formally,  
$\mathcal{U} = (\mathcal{U}_1 -
(\mathcal{U}_2 \cup \mathcal{V}_2)) \cup (\mathcal{U}_2 -
(\mathcal{U}_1 \cup \mathcal{V}_1)) \cup \{C : \exists
i \in\{1,2\} (C\in{U_i} \textrm{ and } M_i \succeq_C
M_{3-i}\}$ and $\mathcal{V} = (\mathcal{V}_1 -
(\mathcal{U}_2 \cup \mathcal{V}_2))) \cup (\mathcal{V}_2 -
(\mathcal{U}_1 \cup \mathcal{V}_1)) \cup \{C: \exists
i \in\{1,2\} (C\in{V_i} \textrm{ and } M_i \succeq_C
M_{3-i})\}$. 
\item For 
$C \in
(\mathcal{U}_1 \cup \mathcal{V}_1)-(\mathcal{U}_2 \cup \mathcal{V}_2)$,
set $\mathcal{R}(C) = \mathcal{R}_1(C)$, $\mathcal{F}(C)
= \mathcal{F}_1(C)$, and $\mathcal{G}(C) = \mathcal{G}_1(C)$.
\item Similarly, for $C \in
(\mathcal{U}_2 \cup \mathcal{V}_2)-(\mathcal{U}_1 \cup \mathcal{V}_1)$,
set $\mathcal{R}(C) = \mathcal{R}_2(C)$, $\mathcal{F}_C
= \mathcal{F}_2(C)$, and $\mathcal{G}(C) = \mathcal{G}_2(C)$.
\item
For
$C \in (\mathcal{U}_1 \cup \mathcal{V}_1) \cap
(\mathcal{U}_2 \cup \mathcal{V}_2)$,
we must have either $M_1 \succeq_C M_2$ or $M_2 \succeq_C M_1$.  If
$M_i \succeq M_{3-i}$, then set
$\mathcal{R}(C) = \mathcal{R}_i(C)$, $\mathcal{F}(C)
= \mathcal{F}_i(C)$, and $\mathcal{G}(C) = \mathcal{G}_i(C)$.
(By Proposition~\ref{pro:antisymmetry}, this is well defined: if
$M_1 \succeq_C M_2$ and 
$M_2 \succeq_C M_1$,  then $\R_1(C) = \R_2(C)$, $\F_1(C) = \F_2(C)$,
and $\G_1(C) = \G_2(C)$.)
\end{itemize}

Returning to Example~\ref{ex1}, it is easy to check that the models
$M_1$ and $M_2$ are compatible.
Specfically, we have $M_1 \succeq_C M_2$ and $M_2 \succeq_B M_1$. (For
all other variables $D$, we have $M_1 \succeq_D M_2$ and
$M_2 \succeq_D M_1$.)
Thus, we can combine $M_1$ and $M_2$ to get the model $M_1 \oplus' M_2$ depicted
in Figure~\ref{fig:incompatibleachexample1}

\begin{figure}[htb]
\vspace{-.1in}
\centering
	\includegraphics[width=1.0\linewidth]{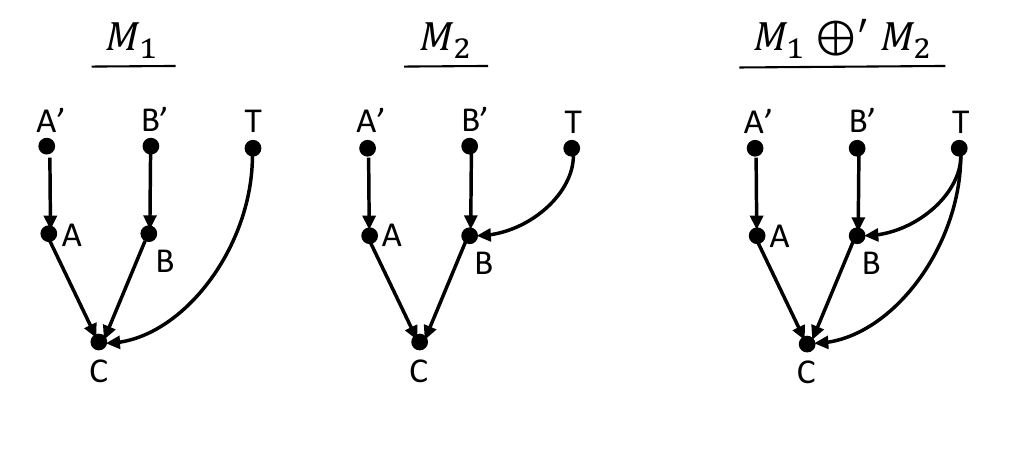}
\vspace{-.3in}
\caption{Taking into account what each scientist focused on allows us
to combine $M_1$ and $M_2$ to get the model shown on the right.}
	\label{fig:incompatibleachexample1}
\end{figure}

In the ACH approach, $M_1$ and $M_2$ 
would be declared incompatible.  Since no
information about who considered what possibilities is available, 
expert $2$ is assumed to have come to the conclusion that $T$ does not
directly affect $C$, and therefore be in fundamental disagreement
with expert $1$.  In our setting, though, we can take advantage of
the focus information to determine whether there is truly a fundamental
disagreement.  The disagreement may just be a result of the fact that
one of the experts was not focusing on certain variables. 
In situations where this information is available, it can allow us
to take more complete advantage of the different areas of expertise
that different experts may have. 

Slightly more generally, the ACH definition 
is designed to take into account situations where one
expert's model is more detailed in terms of the topology of the causal
graph, that is, where the causal relationship is considered to be
mediated by variables the other modeler was simply not aware existed.
In our setting there is more information available, allowing
us to consider another sense in which one modeler's understanding
might be locally more detailed than another's, namely, situations where one
expert can explain the other's observed results by taking into account
the fact that the other was not focusing on certain variables.


This notion of combination is commutative and, when defined, associative:

\begin{proposition}\label{pro:oplusprop}
Given three pairwise compatible models $M_1$, $M_2$, and $M_3$, 
	\begin{enumerate}[(a)]
\commentout{
\item For all $C \in (\mathcal{V}_1 \cup \mathcal{V}_2)$, if $M_1 \succeq_{C} M_2$ and $M_2 \succeq_{C} M_1$ then $\mathcal{F}_1(C)=\mathcal{F}_2(C)$ and $\mathcal{G}_1(C)=\mathcal{G}_2(C)$.
		\item $M_1 {\oplus'} M_2$ is well defined.
}
		\item $M_1 {\oplus'} M_2 = M_2 {\oplus'} M_1$;
\item if $M_1 \succeq_C M_2$ or $C \in (\U_1 \cup \V_1) - (\U_2
\cup \V_2)$, then $M_1 \oplus' M_2 \equiv_C M_1$;
\item if $M_3$ is compatible with $M_1 {\oplus'} M_2$ and $M_1$ is
compatible with $M_2 {\oplus'} M_3$ then $M_1 {\oplus'} (M_2
{\oplus'} M_3) = (M_1 {\oplus'} M_2) {\oplus'} M_3$.
\end{enumerate}
\end{proposition}

\begin{proof}
\commentout{
\begin{enumerate}[(a)]
\item This follows from the weak antisymmetry of $\succeq_C$, proved above.
		\item What we need to show is that if $M_1 \succeq_C M_2$ and $M_2 \succeq_C M_1$ then $M_1$ and $M_2$ are  $C$-equivalent.  We show this in four parts: i) If $M_1 \succeq_C M_2$ and $M_2 \succeq_C M_1$ then $C \in (\mathcal{U}_1 \cup \mathcal{V}_1) \cap (\mathcal{U}_2 \cup \mathcal{V}_2)$, so by the definition of compatibility we have that $\mathcal{R}_1(C)=\mathcal{R}_2(C)$.  ii) If $M_1 \succeq_C M_2$ and $M_2 \succeq_C M_1$ then $\mathcal{F}_1(C) = \mathcal{F}_2(C)$ by part (a).  iii) If $M_1 \succeq_C M_2$ and $M_2 \succeq_C M_1$ then $\mathcal{G}_1(C) = \mathcal{G}_2(C)$ by part (a). iv) Finally, we need to show that if $M_1 \succeq_C M_2$ and $M_2 \succeq_C M_1$ then it cannot be the case that $C$ is endogenous in one and exogenous in the other.  But we have already shown this above in our proof of the weak antisymmetry property.

		\item By our definition, nothing depends on which
		model comes before the operator and which after.
                }
Commutativity is immediate from the definition of $\oplus'$.

For part (b), suppose that $M_1 \succeq_C M_2$.  
Then $C$ is exogenous in $M_1 \oplus' M_2$ iff $C$ is exogenous
in $M_1$.  Moreover, $\R^{M_1 \oplus' M_2}(C) = \R_1(C)$,
$\F^{M_1 \oplus' M_2}(C) = \F_1(C)$ if $C \in \V_1$, and
$\G^{M_1 \oplus' M_2}(C) = \G_1(C)$, so it immediately follows that
$M_1 \oplus' M_2 \equiv_C M_1$.  A similar argument applies if $C \in
(\U_1 \cup \V_1) - (\U_2 \cup \V_2)$.  

For part (c),
observe that to show that
$M_1 \oplus' (M_2 \oplus' M_3) = (M_1 \oplus'
M_2) \oplus' M_3$, it suffices to show that 
$M_1 \oplus' (M_2 \oplus' M_3) \equiv_C (M_1 \oplus'
M_2) \oplus' M_3$ for all $C \in
(\U_1 \cup \V_1 \cup \U_2 \cup \V_2 \cup \U_3 \cup \V_3)$.  
We do this by considering, for each variable $C$, how many models it
appears in.

First consider the case where $C$ is in only one of the three models
(i.e., $C \in U_i \cup \V_i$ for exactly one $i \in \{1,2,3\}$).
Assume without loss of generality that $C$ is in
$M_1$.  Then it follows almost immediately from our
definitions that $M_1 {\oplus'} (M_2 {\oplus'}
M_3) \equiv_C M_1 \equiv_C (M_1 {\oplus'} M_2) {\oplus'}
M_3$, so  $M_1
{\oplus'} (M_2 {\oplus'} M_3) \equiv_C (M_1 {\oplus'} M_2) {\oplus'} M_3$.
Similarly, in the case where $C$
is only in two models, assume without loss of
generality that $C$ is in $M_1$ and $M_2$.  
\commentout{
Then it
follows immediately that $M_2 {\oplus'} M_3 \equiv_C
M_2$, and $(M_1 {\oplus'} M_2) \oplus' M_3 \equiv_C
		M_1 {\oplus'} M_2$, so $M_1 \oplus' (M_2
		{\oplus'} M_3) \equiv_C M_1 \oplus' M_2 \equiv_C 
		(M_1 {\oplus'} M_2) {\oplus'} M_3$. 
}
Then it follows immediately that $(M_1 {\oplus'} M_2) \oplus' M_3 \equiv_C
M_1 {\oplus'} M_2$, so if $M_1 \succeq_C M_2$ then $(M_1 {\oplus'}
		M_2) \oplus' M_3 \equiv_C M_1$ and if $M_2 \succeq_C
		M_1$ then $(M_1 {\oplus'} M_2) \oplus' M_3 \equiv_C
		M_2$.  It is also immediate that $M_2 \oplus'
		M_3 \equiv_C M_2$, so if $M_1 \succeq_C M_2$ then
$M_1 \oplus' (M_2 \oplus' M_3) \equiv_C M_1$.  Now
		consider the case where $M_2 \succeq_C M_1$.   
It must be the case that either $M_2 \oplus' M_3 \succeq_C M_1$ or $M_1 \succeq_C M_2 \oplus' M_3$ because they are compatible.
If $M_2 \oplus' M_3 \succeq_C M_1$ then $M_1 \oplus' (M_2 \oplus'
		M_3) \equiv_C M_2$ and we are done.  On the other
		hand, if $M_1 \succeq_C M_2 \oplus' M_3$ then, because
		$M_2 \oplus' M_3 \equiv_C M_2$, we know that
		$M_1 \succeq_C M_2$.  But then because we assumed
		$M_2 \succeq_C M_1$ we get by
		Proposition~\ref{pro:antisymmetry} that $M_2 \equiv_C
		M_1$ and so $M_1 \oplus' (M_2 \oplus' M_3) \equiv_C
M_1 \equiv_C M_2$.  Thus, in al cases, we have that
		$(M_1 {\oplus'} M_2) \oplus' M_3 \equiv_C M_1 \oplus'
		(M_2 \oplus' M_3)$. 

\commentout{
Finally, if $C$ is in all three models, by
Propositions~\ref{pro:preorder} and \ref{pro:antisymmetry},
for some choice of $i$, $j$, $k$ we have
$M_i \succeq_C M_j \succeq M_k$.  Suppose that $M_1 \succeq_C
M_2 \succeq_C M_3$ (the argument is almost identical in all other
cases).  
It easily follows from part (b) and Proposition~\ref{pro:antisymmetry}
that $(M_1 \oplus' M_2) \oplus' M_3 \equiv_C M_1 \equiv_C M_1 \oplus'
(M_2 \oplus' M_3)$.  This completes the argument.
}

Finally, if $C$ is in all three models, by
Propositions~\ref{pro:preorder} and \ref{pro:antisymmetry}, 
for some choice of $i$, $j$, $k$ we have
$M_i \succeq_C M_j \succeq M_k$.  Suppose that $M_1 \succeq_C
M_2 \succeq_C M_3$ (the argument is almost identical in all other
cases).  
It follows from part (b) that $M_1 \oplus' M_2 \equiv_C M_1$.  Because
$M_3$ is compatible with $(M_1 \oplus' M_2)$, we know that either
$(M_1 \oplus' M_2) \succeq_C M_3$ or $M_3 \succeq_C (M_1 \oplus'
M_2)$.  In the first case, it follows immediately from part (b) that
$(M_1 \oplus' M_2) \oplus' M_3 \equiv_C M_1$.  In the second case, since
$(M_1 \oplus' M_2) \equiv_C M_1$ by part (b) and $M_3 \succeq_C
(M_1 \oplus' M_2)$ by assumption, it follows that $M_3 \succeq_C M_1$.
And since $M_1 \succeq_C M_2 \succeq_C M_3$, we have that 
$M_1 \succeq_C M_3$ by transitivity (Proposition~\ref{pro:preorder}),
so it follows from 
Proposition~\ref{pro:antisymmetry} that $M_3 \equiv_C M_1$.  But then
from part (b) we have that $(M_1 \oplus' M_2) \oplus' M_3 \equiv_C
M_3 \equiv_c M_1$. 
It is easy to show by similar reasoning that $M_1 \oplus' (M_2 \oplus'
M_3) \equiv_C M_1$.  So we get that $M_1 \oplus' (M_2 \oplus'
M_3) \equiv_C (M_1 \oplus' M_2) \oplus' M_3$, completing the
argument. 
%
\end{proof}

One natural question to ask is whether this definition of combination
is guaranteed to preserve acyclicity.  Unfortunately,
this is not the case, as the following example shows.
\begin{example}
{\rm 
Consider the models $M_1$ and $M_2$ in Figure \ref{fig:cyclicmodelsexample1},
where
\begin{itemize}
\item $\U_1=\{A,C\}$, $\V_1 = \{B,D\}$, $\U_2 = \{B,D\}$, and $\V_2 = \{A,C\}$;
\item all variables are binary (i.e.\ have range $\{0,1\}$);
\item $\G_1(C) = \G_1(A) = \G_2(B) = \G_2(D) = \emptyset$,
$\G_1(B) = \{A\}$, $\G_1(D) = \{C\}$, $\G_2(A) = \{D\}$, and $\G_2(C)
= \{B\}$;
\item 
in $M_1$,
$A$ is the parent of $B$ and $C$ is the parent of $D$, and in $M_2$,
$B$ is the parent of $C$ and $D$ is the parent of $A$.
The details of the equations do not matter;
for simplicity, suppose that in $M_1$ we have $B=A$ and $D=C$,
while in $M_2$ we have $A=D$ and $C=B$.
\end{itemize}
Thus, $A$ and $C$ are exogenous in $M_1$, while $B$ and $D$ are
exogenous in $M_2$.  
It is easy to see that,
despite the fact that both models are acyclic, when we combine them we
get a cyclic model.
}
\begin{figure}[htb]
	\centering
\includegraphics[width=1.0\linewidth]{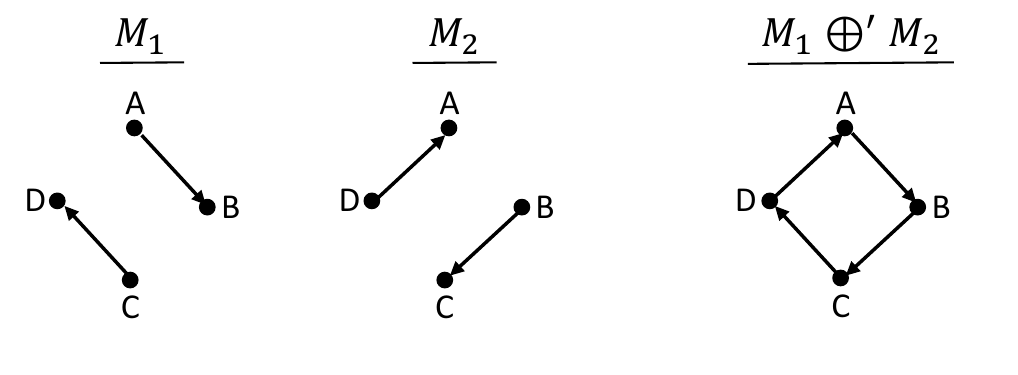}
\caption{Although the models $M_1$ and $M_2$ are acyclic, the combined
model $M_1 \oplus M_2$ contains a cycle.} 
	\label{fig:cyclicmodelsexample1}
\end{figure}
\wbox
\end{example}

We can, however, provide a simple and efficient test to guarantee that
the combined model will be acyclic.  Let $G_1 =
(\mathcal{U}_1 \cup \mathcal{V}_1, E_1)$ be the parent graph for
model $M_1$ and let $G_2 = (\mathcal{U}_2 \cup \mathcal{V}_2, E_2)$ be
the parent graph for model $M_2$.  Let $G' =
((\mathcal{U}_1 \cup \mathcal{V}_1)\cup(\mathcal{U}_2 \cup \mathcal{V}_2),
E_1 \cup E_2)$.  In linear time, we can compute whether $G'$ contains
any cycles.  If it does not, then $M_1 {\oplus'} M_2$ is guaranteed
to be acyclic.  This is a sufficient but not necessary condition for
acyclicity, as edges can be deleted via our combination process.  In
practice, though, we suspect this condition will hold in most cases of
interest where the combined model is indeed acyclic. 


%

\subsection{Combination as least upper bound}

When we combine two models, we would like the combined
model to be the simplest model that can explain both.  
Unfortunately, this may not be the case for
$M_1 \oplus' M_2$.  Indeed, even if $M_1$ and $M_2$ are compatible,
$M_1 \oplus' M_2$ may not be able to 
explain both $M_i$ for all variables $C$ that appear in $M_i$.  It
follows from Proposition~\ref{pro:oplusprop} that if $M_i \succeq_C
M_{3-i}$ or $C \in (\U_i \cup \V_i) - (\U_{3-i} \cup \V_{3-i})$, then
$M_1 \oplus' M_2 \equiv_C M_i$, so (by
Proposition~\ref{pro:antisymmetry}) $M_1 \oplus' M_2$ can explain
$M_i$ with respect to $C$.  But, as the following example shows, if
$M_1 \succeq_C M_2$ and $C$ appears in $M_2$, $M_1 \oplus' M_2$ may
not be able to explain $M_2$ with respect to $C$.

\begin{example}\label{ex:notlub}
{\rm Consider the models $M_1$ and $M_2$ depicted in
Figure \ref{fig:nolubexample1}, where the range of all variables is
$\{\mathsf{true},\mathsf{false}\}$;
the focus set of each variable consists of just its parents, as defined in
the parent graph; 
in $M_1$, the structural equations are such that $C =
A \, \mathsf{XOR} \, B$, while in $M_2$,
$A = D$  and $ B = D$.
Then in $M_1 \oplus' M_2$, all three of these equations hold.
}
\begin{figure}[htb]
	\centering
\includegraphics[width=1.0\linewidth]{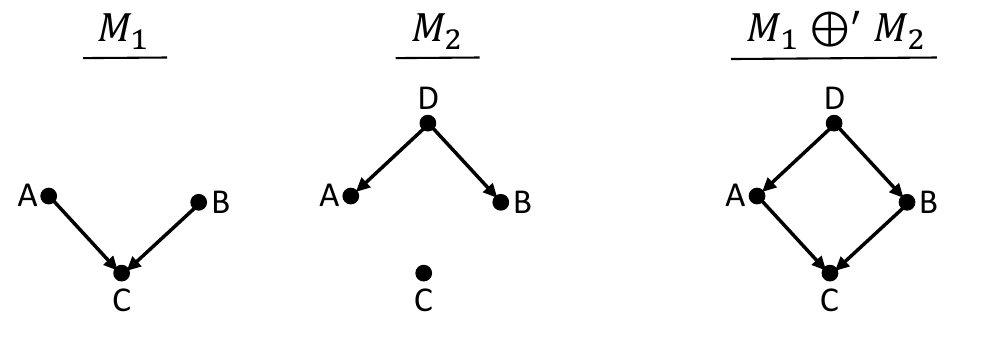}
\caption{Models $M_1$ and $M_2$ where $M_1 \oplus' M_2 \,
{\not\succeq}_C \, M_2$.}
	\label{fig:nolubexample1}
\end{figure}

{\rm It is easy to see that $M_1 \oplus' M_2 \, {\not\succeq}_C \, M_2$.
The problem is, to explain the value $C=0$, $A$ and $B$ need to have
different values, and there is no context in
$M_1 \oplus' M_2$ that gives them different values.
Intuitively, although 
$M_2$ can explain $M_1$ wth respect to each of $A$ and $B$
individually, it cannot explain them both together.  In particular,
the setting $A = \mathsf{false}$ and $B = \mathsf{true}$ cannot be
explained in $M_2$.  We have not defined what 
it would mean to explain a setting involving more than one variable;
this is because our intuition for ``can explain'' is based on the
assumption that experts are testing one variable at a time.

This example also shows that $\succeq_C$ is not necessarily
transitive: we have $M_1 \oplus' M_2 \succeq_C M_1$ and $M_1 \succeq_C
M_2$, we do not have $M_1 \oplus' M_2 \succeq_C M_2$.  This does not
contradict Proposition~\ref{pro:preorder}, since $M_1 \oplus' M_2$ and
$M_2$ are not compatible. }
\wbox
\end{example}

The fact that $M_1 \oplus' M_2$ may not be able to explain both $M_1$
and $M_2$ is somewhat disconcerting.  However, the situation is not quite 
as bad as it appears.  

\begin{definition}\label{dfn:dominance}
$M_1$ \emph{dominates} $M_2$, written $M_1 \succeq M_2$, if 
$M_1 \succeq_C M_2$ for all $C \in \mathcal{U}_2 \cup \mathcal{V}_2$.
\end{definition}
Note that if $M_1 $ dominates $ M_2$, then we must have that
$\mathcal{U}_1 \cup \mathcal{V}_1 \supseteq \mathcal{U}_2 \cup \mathcal{V}_2$. 

\begin{theorem}\label{thm:lub}
If $M_1$ and $M_2$ are compatible, then
$M_1 \oplus' M_2$ dominates both $M_1$ and $M_2$ iff
$M_1 \oplus' M_2$ is the unique least upper bound of $\{M_1,M_2\}$.
\end{theorem}

\begin{proof}
Suppose that
$M_1 {\oplus'} M_2 \succeq M_1$
and $M_1 {\oplus'} M_2 \succeq M_2$.  Then, by definition, $M_1
{\oplus'} M_2$ is an upper bound of $\{M_1,M_2\}$, so now we must show
that, for any other upper bound $M'$ of $\{M_1,M_2\}$, we have 
$M' \succeq M_1 {\oplus'} M_2$.  We first note that the variables in
$M_1 {\oplus'} M_2$ are precisely
$(\mathcal{U}_1 \cup \mathcal{V}_1)\cup(\mathcal{U}_2\cup\mathcal{V}_2)$.
For each variable $C$ in $M_1 {\oplus'} M_2$, there exists some
$i \in \{1,2\}$ such that 
$\mathcal{G}^{{M_1 \oplus'} M_2}(C) = \mathcal{G}_{i}(C)$
 and either $\mathcal{F}^{M_1 {\oplus'} M_2}(C)
= \mathcal{F}_{i}(C)$ or $C$ is exogenous in both $M_1 \oplus' M_2$
and $M_i$.
Moreover, $Par^{M_1\oplus' M_2}(C) = Par_i(C)$.
Thus, given an intervention $\G^{M_1 \oplus' M_2}(C) = \vec{x}$ in
 $M_1 \oplus' M_2$ and context $\vec{u}$ such that $(M_1 \oplus'
 M_2,\vec{u}) \models [\G^{M_1 \oplus' M_2}(C) = \vec{x}](C=c)$, there
 exists a context $\vec{u}'$ in $M_i$ such that
 $(M_i,\vec{u}') \models [\G^{M_1 \oplus' M_2}(C) = \vec{x}](C=c)$.
Since $M' \succeq M_i$, it follows that there exists a context
 $\vec{u}''$ in $M'$ such that
  $(M',\vec{u}'') \models [\G^{M_1 \oplus' M_2}(C) = \vec{x}](C=c)$.
  It follows that $M' \succeq_C M_1 \oplus' M_2$.  Since $C$ was
arbitrary, it follows that $M' \succeq M_1 \oplus' M_2$.

We have thus shown that $M_1 \oplus' M_2$ is a least upper bound of
$\{M_1,M_2\}$ if $M_1 \oplus' M_2 \succeq M_1$ and $M_1 \oplus'
M_2 \succeq M_2$.
Uniqueness is straightforward: if $M'$ is another least upper
bound of $\{M_1,M_2\}$ then, by 
Proposition~\ref{pro:antisymmetry}, it follows that 
$M' \equiv_C M_1 \oplus' M_2$ for all
$C \in \U_1 \cup \V_1 \cup \U_2 \cup \V_2$, so
$M' = M_1 \oplus' M_2$.
The converse is also immediate: if
$M_1 \oplus M_2$ is not an upper bound of both $M_1$ and $M_2$, it
certainly cannot be a least upper bound of $\{M_1,M_2\}$.
\end{proof}

So where does this leave us?
Our goal is to combine the information of experts.  If a
decision-maker believes that models $M_1$ and $M_2$ both provide useful 
information, then she would want to work with a model that somehow
combines this information.  As Example~\ref{ex:notlub} shows, the
problem with $M_1 \oplus' M_2$ is that it does not necessarily combine
all the information in $M_1$ and $M_2$.
To deal with this problem, we simply 
define $\oplus$ by taking $M_1 \oplus M_2 = M_1 \oplus' M_2$ if
$M_1 \oplus' M_2 \succeq M_i$ for $i = 1,2$, and otherwise say that
$M_1$ and $M_2$ are incompatible and 
$M_1 \oplus M_2$ is undefined.  Intuitively, in the latter case,
there is no clear way to explain
both models,  so more experiments are necessary. 
It is easy to check
that Proposition~\ref{pro:oplusprop} holds for $\oplus$, with no
change in proof.  Moreover, by Proposition~\ref{thm:lub}, when it is
defined, $M_1 \oplus M_2$ is the least upper bound of $\{M_1,M_2\}$.

We conjecture that if $M_1 \oplus M_2$ is not defined, then
$\{M_1,M_2\}$ in fact has no least upper bound.  This is the case in
the models of Example~\ref{ex:notlub}.  Consider the models $M_1'$ and
$M_2'$, where $M_1'$ is identical to $M_1$ except that it includes the
variable $D$, and $\G^{M_1'}(A) = \G^{M_1'}(B) = \{D\}$, and $M_2'$ is
just like $M_2$ except that $\G^{M_2'}(C) = \{A,B\}$.  It is easy to
check that $M_1$ and $M_2$ are both upper bounds on $\{M_1,M_2\}$, and
there is no upper bound $M'$ of $\{M_1,M_2\}$ such that $M_1' \succeq M'$
and $M_2' \succeq M'$.

If this conjecture is 
correct (and we have shown that it is in a number of special cases),
then it shows that if we think of $\succeq$ as an information
ordering, then $M_1 \oplus M_2$, when it is defined, is the model that
combines the information in $M_1$ and $M_2$ and has no additional
information; if it is not defined, then there is no such model.%
\footnote{We remark that we can define an analogue of $\succeq$
for the notion of combination considered by ACH, and show that
$M_1 \oplus M_2$ as ACH define it is the least upper bound $M_1$ and
$M_2$ with respect to the ACH notion.  Thus, thinking in terms of
least upper bound seems like a useful way to think of combining models.}


\subsection{Explanation complexity and combination complexity}

\commentout{
Our definitions thus far capture the notion that one
model can explain another with respect to $C$.  These
explanations, though, may in fact be quite complicated.  Intuitively,
we may be more reluctant to accept explanations that are
complicated; if a highly complicated explanation is needed
to reconcile two models,c we may instead prefer to simply declare them
incompatible. 
The following definitions of \textit{explanation complexity}
and \textit{combination complexity} capture these intuitions.}

Recall that 
$M_1 \succeq_C M_2$ if, for every intervention
$\G_2(C) = \vec{x}$, value $c \in \R(C)$, and context
$\vec{u}_2$, there exists a context $\vec{u}_1$ such that
if $(M_2, \vec{u}_2) \vDash
[\mathcal{G}_2(C) \leftarrow \vec{x}]( C=c)$ then
$(M_1, \vec{u}_1) \vDash [\mathcal{G}_2(C) \leftarrow \vec{x}] (C=c)$. 
However, in principle,
we could use a different context $\vec{u}_1$ to explain each
possible intervention on $\G_2(C)$.
We might be reluctant to accept explanations that are
complicated, in the sense of requiring too many different contexts; if
an overly complicated explanation is needed 
to reconcile two models, we may instead prefer to simply declare them
incompatible. 
The following definitions of \textit{explanation complexity}
and \textit{combination complexity} capture these intuitions.

\begin{definition} $M_1$
\emph{can explain
$M_2$ with respect to $C$ using a set $\U'$ of contexts} if $M_1$
can explain $M_2$ with respect to $C$ using only contexts $u'_1$ drawn
from $\U_1$; that is, we just modify Definition~\ref{dfn:canexplain}
so that all the contexts $u_1$ in condition (c) are drawn from $\U'$.
The \emph{complexity of $M_1$'s ability to explain $M_2$ with
respect to $C$} is $\min\{|\U'|:$ $M_1$ can explain $M_2$ with respect
to $C$ using $\U'\}$.   
\end{definition}

\begin{example}
{\rm Consider the models in
Figure \ref{fig:complexityexample1}.  In all of these models, 
$\mathcal{R}(A) = \mathcal{R}(B) = \mathcal{R}(A_1) = \mathcal{R}(A_2)
= \mathcal{R}(A_3) = \{0,\dots,10\}$, $\mathcal{R}(D)
= \{0,\dots,30\}$, and $\mathcal{R}(C) = \{0,\dots,60\}$.
In 
model $M_1$ on the left, we have the structural
equations $C = A+B$ if 
$D\geq{1}$ and 
$C= 2(A+B)$ if $D=0$;  in
model $M_2$, we have  $C = A+B$; in
model $M_3$ on the right, we have 
$C=D$; and in model $M_4$, we have $C =
A_1 + A_2 + A_3$.
In the low-complexity models on the left,
the complexity of
$M_1$'s ability to explain $M_2$ with respect to $C$ is $1$, as every
intervention can be explained by the value of $D$ simply having been
$1$ the entire time.  For the high-complexity models on the right, though, the
complexity of $M_3$'s ability to explain $M_4$ with respect to $C$ is
$30$; for each intervention, $D$ must take on precisely the
right value in $M_3$ for each particular outcome of $C$ to be observed.  Thus,
we would be more hesitant to combine the high-complexity models $M_3$
and $M_4$.  Combining them 
implicitly assumes that $M_3$ and $M_4$ are compatible, and, in
particular, that $M_3$ can explain $M_4$ with respect to $C$.
}
\begin{figure}[htb]
	\centering
\includegraphics[width=1.0\linewidth]{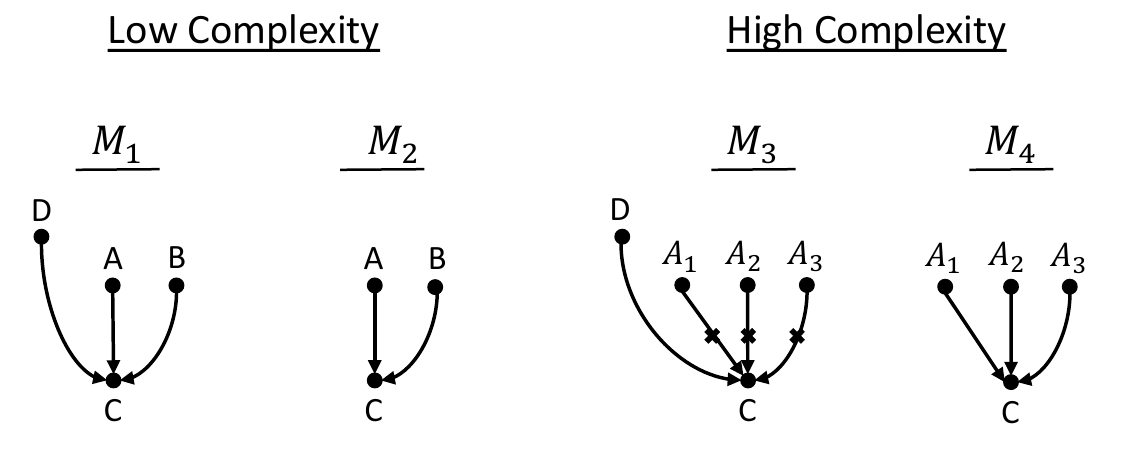}
\caption{The two models on the left have low explanation
	complexity with respect to $C$ whereas the two on the right
	have high explanation complexity.} 
	\label{fig:complexityexample1}
\end{figure}
\wbox
\end{example}

We can extend the notion of explanation complexity to the combination 
complexity of two models.

\begin{definition}\label{dfn:complexity}
The \emph{combination 
	complexity of two compatible models $M_1$ and $M_2$} is 
the minimum cardinality $|\U'|$ taken over all sets  $\U'$ such that,
for all $C \in
(\mathcal{U}_1 \cup \mathcal{V}_1)\cap(\mathcal{U}_2 \cup \mathcal{V}_2)$, 
either
$M_1$ can explain $M_2$ with respect to $C$ using $\U'$ or $M_2$
can explain $M_1$ with respect to $C$ using $\U'$.
\end{definition}

A decision-maker may want to consider only explanations that have
complexity less than or equal to some threshold or model combinations that have complexity less than a threshold.  In the next
section, we show how combination complexity can be used to
weight models.

\section{Weighting and Combining Expert Opinions}
\label{sec:weight}

Given a collection of models, it may be impossible to combine all of
them, but possible to combine a variety of different subsets of them.
ACH proposed a way to assign confidence to different
possible combined models based on the decision-maker's confidence in
the original models.  Here we provide a way to extend this to our
setting. 

We start with a collection of pairs $(M_1,p_1),\dots,(M_n,p_n)$ where
$M_i$ is a causal model with focus and $p_i$ is a value in $(0,1]$.
Here the intuition for each pair should be that $M_i$ was the model
proposed by expert $i$ and $p_i$ is the decision-maker's degree of
confidence that 
expert $i$'s model is correct. More precisely, $p_i$ is not the
decision-maker's degree of confidence that the 
assumptions built into $M_i$ are correct, but her confidence that, for each
variable $C$ and intervention $\G_i(C) = \vec{x}$, if
$(M_i,\vec{u})\vDash [\G_i(C) = \vec{x}]( C=c)$ then expert $i$ 
indeed observed a world where the variables in $\G_i(C)$ were
$\vec{x}$ and $C$ did have value $c$.  Following ACH, we define
$Compat = \{I \subseteq \{1,\dots,n\} : \textrm{ the models in } \{M_i:
 i\in{I}\} \textrm{ are mutually compatible} \}$ and define $M_I =
{\oplus}_{i\in{I}} M_i$ for all $I \in Compat$.  The \emph{mutual
compatibility} 
of a set $\mathcal{M}$ of models is defined inductively on the
cardinality of $\mathcal{M}$.  If $|\mathcal{M}|=1$ then $\mathcal{M}$
is automatically mutually compatible, and if $|\mathcal{M}|=2$ then
$\mathcal{M}$ is mutually compatible if the two models in
$\mathcal{M}$ are compatible.  If $|\mathcal{M}|=n$ then $\mathcal{M}$
is mutually compatible if every subset of cardinality $n-1$ is
mutually compatible and, for each $M \in \mathcal{M}$, $M$ is
compatible with $\oplus_{M' \in \mathcal{M} : M' \neq M} M'$.  

One simple way to weight the combined models, proposed
by ACH, is to assign model $M_I$ probability
\begin{equation}\label{eq:combine}
p_I = \displaystyle\prod_{i\in{I}}p_i
* \displaystyle\prod_{j\notin{I}}(1-p_j) / N,
\end{equation}
where $N$ is simply a
normalization term to get the probabilities to sum to $1$.
Thus, $p_I$ captures the intuition that the agents in $I$ performed
their experiments 
correctly while the agents not in $I$ may have made a mistake in one
or more of their experiments, where the probabilities of agents having
made a mistake are treated as being mutually independent. 

Let $\M_I = \{M_i: i \in I\}$. 
In our setting, we may also want to take into account how complex
it is to combine the models in $\M_I$ when assigning $M_I$ a
probability; if combining the models in $\M_I$ requires a large set of
contexts to make all of the necessarily explanations, then we may have less
confidence that the combined model captures the true state of the
world.  To formalize this idea, we first generalize 
Definition~\ref{dfn:complexity} in the obvious way:
the combination complexity of a set $\mathcal{M}$ is the
minimum cardinality $|\U'|$ of a set $\U'$ such that all explanations
made during the combination process can be made using $\U'$.
The combination complexity of a singleton set is defined to be $1$.

Exactly how complexity should be taken into account when assigning
confidence scores may be context-dependent; it is up to the
decision-maker who is combining the models to decide.  We propose
several simple rules here. One simple rule that
may be relevant in some situations is to simply use a threshold, and assign
confidence $0$ to models where the combination complexity or
the explanation complexity with respect to any variable $C$ is above
some constant $\mu$.  (Here and 
in the following two rules, the normalization factor $N$ must be
updated accordingly.)  Another natural option may be to add a
weighting factor to (\ref{eq:combine}) that is inversely proportional
to the combination complexity.
If the combination complexity
of $\mathcal{M}_I$ is $\mu_I$, then we set $$p_I'
= \frac{1}{\mu_I}*\displaystyle\prod_{i\in{I}}p_i
* \displaystyle\prod_{j\notin{I}}(1-p_j) / N.$$ A third rule that may
be useful in some contexts is to assign complexity weights that are
inverse exponential in the combination complexity.  Here the
confidence scores assigned would be
$$p_I'' =
e^{-\mu_I}*\displaystyle\prod_{i\in{I}}p_i
* \displaystyle\prod_{j\notin{I}}(1-p_j) / N.$$

\begin{example}
{\rm
Consider three models $M_1$, $M_2$, and $M_3$, where  
\begin{itemize}
\item $\mathcal{U}_1 = \mathcal{U}_3 = \{A,B,D\}$, $\mathcal{V}_1
= \mathcal{V}_3 = \{C\}$, $\mathcal{U}_2 
= \{A,G\}$, and $\mathcal{V}_2 = \{B,C\}$;
\item $\mathcal{R}_{1}(C) = \mathcal{R}_{2}(C) = \mathcal{R}_{3}(C)
= \mathcal{R}_{1}(D) = \mathcal{R}_{3}(D) = \{0,1,2\}$ and
$\mathcal{R}_{1}(A) = \mathcal{R}_{1}(B) = \mathcal{R}_{2}(A)
= \mathcal{R}_{2}(B) = \mathcal{R}_{2}(G) = \mathcal{R}_{3}(A)
= \mathcal{R}_{3}(B) = \{0,1\}$;
\item $\mathcal{G}_1(C) = \{A,B,D\}$, $\mathcal{G}_2(C)
= \{A,B\}$, $\mathcal{G}_2(B) = \{G\}$, and $\mathcal{G}_3(C) = \{A,B,D\}$;
\item the structural equations are such that, in $M_1$, $C=D$; in
$M_2$, $C = A+B$ and $B=G$; and in $M_3$, $C=2$ if 
$D=0$ and $C = \min(1,A+B)$ if $D=1$ or $D=2$.
\end{itemize}
The models in the set $\{M_I: I \in Compat\}$ are $M_1$, $M_2$, $M_3$,
$M_1 \oplus M_2$, and  $M_3 \oplus M_2$,
with combination
complexity $5$ for $M_1 \oplus M_2$ ($3$ for $M_1$ to explain $M_2$
with respect to $C$ and $2$ for $M_2$ to explain $M_1$ with respect to
$B$) and combination complexity $4$ for $M_3 \oplus M_2$ ($2$ for
$M_3$ to explain $M_2$ with respect to $C$ and $2$ for $M_2$ to
explain $M_3$ with respect to $B$).
Of course, $M_1$, $M_2$, and $M_3$ (viewed as singleton sets) all have
combination complexity 1, by definition. 
Consider the second
weighting rule above, inversely proportional weighting, with prior
confidences $p_1=0.85, p_2 = 0.8, \textrm{ and } p_3 = 0.9$.  The
assigned confidence scores would then be 
$$\begin{array}{ll}
\commentout{
p_{M_1}' = (0.85)(0.2)(0.1) / N \approx 0.156 \\
p_{M_2}' = (0.15)(0.8)(0.1) / N \approx 0.110 \\
p_{M_3}' = (0.15)(0.2)(0.9) / N \approx 0.248 \\
p_{M_1 \oplus M_2}' = (\frac{1}{4})(0.85)(0.8)(0.1) / N \approx 0.156 \\
p_{M_3 \oplus M_2}' = (\frac{1}{3})(0.15)(0.8)(0.9) / N \approx
0.330.}

p_{M_1}' = (0.85)(0.2)(0.1) / N \approx 0.176 \\
p_{M_2}' = (0.15)(0.8)(0.1) / N \approx 0.124 \\
p_{M_3}' = (0.15)(0.2)(0.9) / N \approx 0.280 \\
p_{M_1 \oplus M_2}' = (\frac{1}{5})(0.85)(0.8)(0.1) / N \approx 0.141 \\
p_{M_3 \oplus M_2}' = (\frac{1}{4})(0.15)(0.8)(0.9) / N \approx
0.280.
\end{array}$$

Under the third rule, inverse exponential weighting, with the same
prior confidences, the assigned confidence scores would be 
$$\begin{array}{ll}
\commentout{
p_{M_1}'' = (0.85)(0.2)(0.1) / N \approx 0.271 \\
p_{M_2}'' = (0.15)(0.8)(0.1) / N \approx 0.192 \\
p_{M_3}'' = (0.15)(0.2)(0.9) / N \approx 0.431 \\
p_{M_1 \oplus M_2}'' = (e^{-4})(0.85)(0.8)(0.1) / N \approx 0.020 \\
p_{M_3 \oplus M_2}'' = (e^{-3})(0.15)(0.8)(0.9) / N \approx
0.086.}

p_{M_1}'' = (0.85)(0.2)(0.1) / N \approx 0.291 \\
p_{M_2}'' = (0.15)(0.8)(0.1) / N \approx 0.205 \\
p_{M_3}'' = (0.15)(0.2)(0.9) / N \approx 0.462 \\
p_{M_1 \oplus M_2}'' = (e^{-5})(0.85)(0.8)(0.1) / N \approx 0.008 \\
p_{M_3 \oplus M_2}'' = (e^{-4})(0.15)(0.8)(0.9) / N \approx
0.034.
\end{array}$$

As expected, the inverse exponential weighting rule is more complexity
averse, and so assigns a greater proportion of confidence to the
uncombined models. 
%
}
\wbox
\end{example}

These three rules behave in a qualitatively similar manner, with the
importance of complexity being taken into account in different ways. 
More generally, let $\mu_I$ be the combination complexity of
$\mathcal{M}_I$ and let $Q_I = \displaystyle\prod_{i\in{I}}p_i
* \displaystyle\prod_{j\notin{I}}(1-p_j)$.  We believe that there are
many reasonable functions $f(Q_I,\mu_I)$ that can be used to assign 
a confidence scores to $M_I$; we leave it up
to the decision-maker to decide what function $f$ is most suitable for
a given context.
The two requirements that seem necessary to us is that $f$
be non-increasing in $\mu_I$ and non-decreasing in
$Q_I$; that is, $f(Q_I,\mu_I) \geq f(Q_I,\mu_I')$ for fixed
$Q_I$ if $\mu_I' \ge \mu_I$, and $f(Q_I,\mu_I) \leq f(Q_I',\mu_I)$ for
fixed $\mu_I$ if $Q_I' \ge Q_I$.
%
These two rules capture the
intuition that we should not prefer models that are more
complicated, nor should we prefer models that are composed of
models in which we had less prior confidence. 

An additional factor that may sometimes play a role is the likelihood
of different endogenous settings occurring.  If one model can
explain the other only by using a context that is 
very unlikely to occur, then we may not want to assign much weight to
that combined model.  Thus, in certain settings it may also make sense
to have the confidence scores depend on a distribution over exogenous
settings. 

\section{Computational Complexity}\label{sec:complexity}

We now consider the computational complexity of
determining whether one model can explain another with respect to
$C$.

\begin{theorem}
Determining whether $M_1 \succeq_C M_2$ is in $\Pi^{P}_{2}$, and is
$\Pi^{P}_2$-hard, even in
instances where all variables are binary.
\end{theorem}

\begin{proof}
It is easy to see that the problem is in $\Pi^P_2$: the first
two conditions in the can-explain relation can clearly be checked in
polynomial time, while, for a fixed intervention $\G_2(C) = \vec{x}$
in $M_2$, context $\vec{u}_2$ in $M_2$, and context $\vec{u}_1$ in
$M_1$, checking whether  $(M_2, \vec{u}_2) \vDash
[\mathcal{G}_2(C) \leftarrow \vec{x}]( C=c)$ and
$(M_1, \vec{u}_1) \vDash [\mathcal{G}_2(C) \leftarrow \vec{x}] (C=c)$
can be done in polynomial time.

\commentout{
For the lower bound,
consider the canonical $\Pi^{P}_{2}$-hard language
	$\Pi^{P}_{2}(\textrm{SAT})
	= \{\forall\vec{X}\exists\vec{Y}\varphi
: \forall\vec{X}\exists\vec{Y}\varphi \textrm{ is a closed
	quantified}$ $\textrm{ Boolean formula,
	} \forall\vec{X}\exists\vec{Y}\varphi = \textbf{true} \}$. We
	show a reduction from $\Pi^{P}_{2}(\textrm{SAT})$ to our
	language.
			}

For the lower bound,
consider the canonical $\Pi^{P}_{2}$-hard language
$\Pi^{P}_{2}(\textrm{SAT}) $$
= $$ \{\forall\vec{X}\exists\vec{Y}\varphi
$$ : $$ \forall\vec{X}\exists\vec{Y}\varphi \textrm{ is a closed
	quantified}$ $\textrm{ Boolean formula,
} \forall\vec{X}\exists\vec{Y}\varphi = \textbf{true} \}$. We
show a reduction from $\Pi^{P}_{2}(\textrm{SAT})$ to our
language. 

	Consider a CQBF (closed quantified Boolean formula)
        $\forall\vec{X}\exists\vec{Y}\varphi$; we show 
how to transform this into an instance of our problem. For
ease of exposition, we assume that all variables in
	$\vec{X}\cup\vec{Y}$ appear in $\varphi$. Let $M_2$ contain
	exogenous variables $\vec{X}$ and an endogenous variable
	$C\notin \vec{X}\cup\vec{Y}$.  
In $M_2$, the range of all variables
is $ \{\mathsf{true},\mathsf{false}\}$,
$\mathcal{G}_2(C) = \vec{X}$, and the equation for 
$C$ is  $C = \mathsf{true}$.  In $M_1$, 
we have $\U_1 = \vec{X} \cup \vec{Y}$, $\V_1 = \{C\}$, $\G_1(C)
= \vec{X} \cup \vec{Y}$, and the equation for $C$ is $C = \varphi$.
	
	We now show that $M_1 \succeq_C M_2$ if and only
$\forall\vec{X}\exists\vec{Y}\varphi$ is true.
First, suppose that $M_1 \succeq_C M_2$.  Because $C$ is always
	$\mathsf{true}$ in $M_2$ and $\mathcal{G}_2(C) = \vec{X}$,
by condition (c) in the definition of the can-explain relation,
for all settings of $\vec{X}$ there must be a setting of
	the remaining variables such that $C = \mathsf{true}$ in
$M_1$.  But because the equation for $C$ in $M_1$ is $C = \varphi$,
this means that for all settings of $\vec{X}$, there exists a
	setting of $\vec{Y}$ such that $\varphi$ is true.  For the
	other direction, suppose that
        $\forall\vec{X}\exists\vec{Y}\varphi$ is true.
Clearly $\G_2(C) \subseteq \G_1(C)$ and $\R_1(C) = \R_2(C)$.  To see
that condition (c) of the definition of can-explain holds,
consider an intervention
	$\vec{X} = \vec{x}$ on $\vec{X}$.  Because
$\forall\vec{X}\exists\vec{Y}\varphi$ is true,         there must 
	be some setting of the values in $\vec{Y}$ such that if
$\vec{X}$ were set to $\vec{x}$, then $C$ would evaluate to
	$\mathsf{true}$ in $M_1$.  So in the context where
	$\vec{Y}$ is set correspondingly, we get that the original
intervention would make $C = \mathsf{true}$, as desired.  
\end{proof}

While this result indicates that this computation is likely to be
intractable in the worst case, models that arise in the physical and
social sciences often contain only a small number of variables, so we
would still expect these computations to be feasible in practice.

\section{Conclusion}
\label{sec:conclusion}

We have shown how causal models can be combined in
instances where experts disagree due to different focus areas.  We
defined what it means for one model to be able to explain another
with respect to a given variable and showed how this can be used to
combine two compatible models.  Furthermore, we showed that the
model obtained via this combination process is in fact the least upper
bound of the combined models relative to the natural relation, in some
sense making it the simplest model that can explain the observations
of both experts. 

The can-explain relation embodies one way of explaining why two
experts may have different causal models.  ACH can be viewed as
modeling a different reason, 
where $M_1$ is ``better than'' $M_2$ with respect to a variable $C$ in
the ACH view if, roughly speaking, $M_1$ has a more detailed picture
of the causal relations among the ancestors of $C$.
While we believe that the can-explain relation captures quite a
natural intuition (as does the ACH notion of compatibility!), there may
well be other reasonable intuitions that are worth exploring.  More
generally, this viewpoint suggests that a decision-maker trying to
combine experts' models must think seriously about the reasons
underlying their disagreement before combining models, and consider a
notion of combination appropriate for these reasons.  Since the need
to combine expert opinions arises frequently in practice, having a
principled understanding of the process seems to us critical.
We hope that the results of this paper help in this process.

\commentout{
There are also some technical

Furthermore, while we have shown that if two models are compatible
then their combination is the unique least upper bound if one exists,
it can still be the case that there is a least upper bound but the
models are declared incompatible under our definition, such as when
two models share a structural equation for $C$ but have different
focal sets.  It is worth classifying precisely when this occurs and
trying to determine whether there is a more general algorithm that
would be guaranteed to find any least upper bound of a set of models. 

Finally, it is worth carefully examining the ideas proposed by ACH in
light of the current work.  Does the notion of combination they
proposed precisely characterize the least upper bound of two models
relative to the relation they proposed?  \textbf{I'm not certain how
to write this.  Technically it's the natural extension of the relation
proposed to be over models...} If not, what are the characteristics of
the least upper bound in their model?  How, precisely, do least upper
bounds under their definitions relate to least upper bounds under the
definitions proposed here?  

We believe that this and future work on this topic will be valuable
for making formal causal reasoning a useful tool in the natural,
social, and medical sciences.
}

\paragraph{Acknowledgments:}
This work was supported in part by NSF grants IIS-1703846 and IIS-1718108,
ARO grant W911NF-17-1-0592, and a grant from the Open Philanthropy project.

\bibliographystyle{aaai}
\bibliography{joe,z,fv}
\end{document}